\pdfoutput=1

\documentclass[11pt]{article}

\usepackage[]{acl}

\usepackage{times}
\usepackage{latexsym}

\usepackage[T1]{fontenc}

\usepackage[utf8]{inputenc}

\usepackage{microtype}

\usepackage{booktabs}
\usepackage{siunitx}

\usepackage{inconsolata}

\usepackage{xcolor} 
\usepackage{mdframed} 
\definecolor{azureblue}{RGB}{240, 255, 255}
\definecolor{darkcyan}{rgb}{0.0, 0.55, 0.55}
\definecolor{cadmiumgreen}{rgb}{0.0, 0.42, 0.24}

\usepackage{tabularx}

\usepackage{amsmath,graphicx}

\newcommand{\perspective}{\textsc{Perspective API}\xspace}

\newcommand{\HateBERT}{\textsc{HateBERT}\xspace}
\newcommand{\HateCheck}{\textsc{HateCheck}\xspace}

\newcommand{\hatA}{\widehat{A}}
\newcommand{\hata}{\widehat{a}}

\newcommand{\D}{\mathcal{D}}

\renewcommand{\S}{\mathcal{S}}

\newcommand{\TE}{\texttt{TE}}
\newcommand{\ATE}{\texttt{ATE}}
\usepackage{xspace}

\DeclareMathOperator*{\Prob}{\mathbb{P}}
\DeclareMathOperator*{\E}{\mathbb{E}}

\usepackage{amsthm}
\newtheoremstyle{thmstyle}
{0.5em} 
{0.15em} 
{} 
{} 
{\bfseries} 
{.} 
{.5em} 
{} 

\theoremstyle{thmstyle} 
\newtheorem{thm}{Theorem}
\newtheorem*{thm*}{Theorem}
\newtheorem{lem}{Lemma}

\newtheorem*{claim*}{Claim}

\theoremstyle{definition}
\newtheorem{defn}{Definition}

\theoremstyle{remark}

\newtheorem{assumption}{Assumption}

\usepackage{amsmath,amssymb,amsfonts,accents,dsfont}
\usepackage{hyperref}
\usepackage{xcolor,colortbl}

\definecolor{LightCyan}{rgb}{0.88,1,1}
\definecolor{Gray}{gray}{0.95}

\definecolor{DarkBlue}{rgb}{0.1,0.1,0.5}
\definecolor{LightBlue}{rgb}{0.3,0.3,0.7}
\definecolor{DarkGreen}{rgb}{0.1,0.5,0.1}
\definecolor{LightGreen}{rgb}{0.3,0.7,0.3}
\usepackage{hyperref}
\hypersetup{
	colorlinks   = true,
	linkcolor    = LightBlue, 
	urlcolor     = DarkBlue, 
	citecolor    = LightGreen 
}

\usepackage{svg}

\title{Causal ATE Mitigates Unintended Bias in Controlled Text Generation}

\author{Rahul Madhavan \\
  IISc, Bangalore \\
  \texttt{mrahul@iisc.ac.in} \\\And
  Kahini Wadhawan \\
  IBM Research, Delhi \\
  \texttt{kahini.wadhawan1@ibm.com} \\}

\begin{document}
\maketitle

\begin{abstract}
We study attribute control in language models through the method of Causal Average Treatment Effect (Causal ATE). Existing methods for the attribute control task in Language Models (LMs) check for the co-occurrence of words in a sentence with the attribute of interest, and control for them. However, spurious correlation of the words with the attribute in the training dataset,  can cause models to hallucinate the presence of the attribute when presented with the spurious correlate during inference. We show that the simple perturbation-based method of Causal ATE removes this unintended effect.
Specifically, we ground it in the problem of toxicity mitigation, where a significant challenge lies in the inadvertent bias that often emerges towards protected groups post detoxification. We show that this unintended bias can be solved by the use of the Causal ATE metric. We provide experimental validations for our claims and release our code (anonymously) here: \href{https://github.com/causalate-mitigates-bias/causal-ate-mitigates-bias}{github.com/causalate-mitigates-bias/causal-ate-mitigates-bias}.
\end{abstract}


\section{Introduction}
\label{sec:intro}
\vspace{-0.1in}







\begin{figure}[!thb]
    \includegraphics[width=0.48\textwidth]{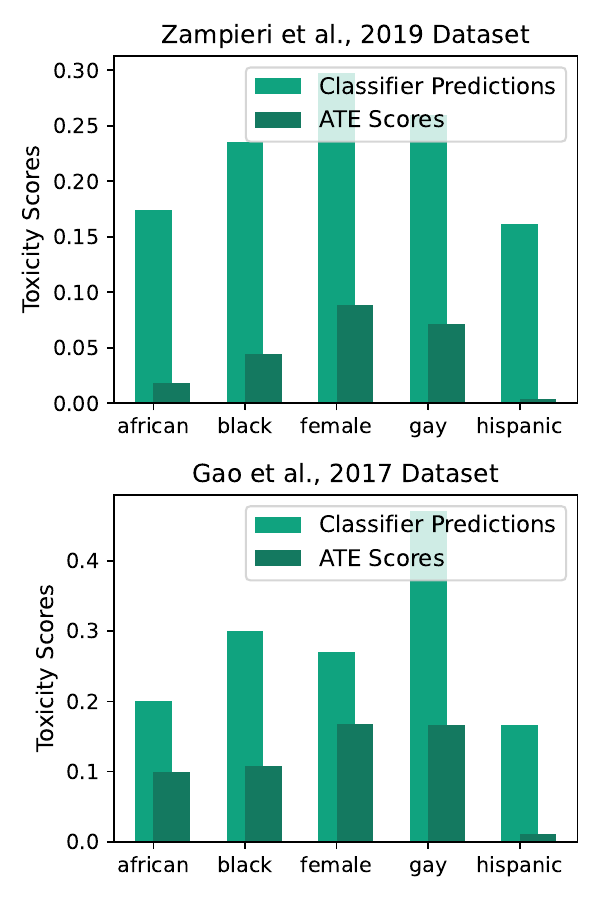}
    \caption{We plot the ATE score vs a regression based classifier for toxicity across two datasets. ATE Scores show a lower toxicity for protected groups.}
    \label{plot:Multiple Datasets}
\end{figure}

Controllable text generation methods are often used to guide the text generated by language models (LMs) towards certain desirable attributes \cite{hu2021causal,dathathri2019plug,liu2021dexperts}. The goal herein is to generate sentences whose attributes can be controlled \cite{prabhumoye2020exploring}.
Language models, which are pre-trained only for next word prediction, cannot directly control for attributes in their outputs. On the other hand, one may wish to alter words in the auto-regressively produced sentences, either accentuating or mitigating the desired attributes.
Attributes such as sentiment, writing style, language precision, tone, and toxicity are key concerns for control in language models, with particular emphasis on toxicity mitigation due to its relevance in sensitive contexts \cite{perez2020rediscovering}.

Regularizers in the reward models are often employed during training to alter the output sentences towards certain desirable attributes  \cite{hu2017toward}. 
Such regularization penalities (or rewards) often rely on models trained on real-world datasets. Such datasets contain spurious correlates -- words that correlate with certain attributes without necessarily causing them \cite{nam2020learning,udomcharoenchaikit2022mitigating}.


In the context of toxicity mitigation, prior works show that detoxification methods inadvertently impact language model outputs concerning marginalized groups \citep{welbl2021challenges}.
Words such as `gay' or `female' are identified as being toxic, as they co-occur with toxic text, and hence the LM stops speaking about them \cite{xu2021detoxifying}. 

This is called the \textit{unintended bias problem}. In this paper we provide experimental and theoretical justifications for the use of causal ATE to mitigate the unintended bias problem in text classification. We prove theoretically that for spurious correlates, the causal ATE score is upper-bounded. We also show through extensive experiments on two popular toxicity classification datasets \cite{zampieri-etal-2019-semeval, gao2017detecting} that our method shows experimental promise (See Figure \ref{plot:Multiple Datasets}).




We provide a full list of related works in Appendix Section \ref{sec:related-work}.

\vspace{-0.05in}

\subsection{Our Contributions:}

\vspace{-0.05in}
\textbf{1.} We show theoretically that the Causal ATE score of spurious correlates is less than $0.25$ under mild assumptions in Sections \ref{sec:notations} and \ref{sec:theory}.\\
\textbf{2.} We provide a theoretical basis for the study of the perturbation based Causal ATE method. We show that it can be used alongside any classifier towards improving it for false positive rates.\\ 
\textbf{3.} We provide experimental validation for our claims by showing that causal ATE scores indeed decrease the toxicity for spurious correlates to toxic sentences in Section \ref{section: experiments}.



\vspace{-0.12in}
\section{Notations and Methodology}
\label{sec:notations}

\vspace{-0.12in}
Consider a sentence $s$, made up of tokens (words) from some universe of words $W$. 
Let the list of all sentences $s$ in our dataset be denoted $\S$. Let each sentence $s \in \S$ be labelled with the presence or absence of an attribute $A$. So the dataset, which we can call $\D$, consists of tuples $(s,A(s))$ for all $s\in\S$. Let the cardinality of the labelled dataset be $|\D| = |\S| = n$.

From such a dataset, it is possible to construct an attribute model that gives us an estimate of the probability of attribute $A$, given a sentence $s$. 
i.e. It is possible to construct a model $\hatA (\cdot)$ such that $\hatA(s) = \widehat{\Prob} \{A \mid s\}$ for any given sentence $s$.
Now such a model may rely on the words in $s$. Let $s = \{w_1,\dots,w_n\}$. We now define an attribute model $\hata (\cdot)$ given a word as follows:

\begin{defn}[Attribute model $\hata (w_i)$ for any word $w_i \in W$]
\label{defn: attribute model for words}
\small
\begin{align}
    \hata(w_i) &:= \frac{|\{\text{sentences } s \in \D \text{ containing }w_i \text{ s.t. } A(s) = 1\}|}{|\{\text{sentences } s \in \D \text{ containing }w_i\}|}\\
    &=\frac{n( A(s) = 1 \mid w_i \in s)}{n(s \mid w_i \in s)}
\end{align}
\normalsize
where $n(\cdot)$ denotes the cardinality of the set satisfying the properties.
\end{defn}

Note that such a model is purely correlation based, and can be seen as the proportion of sentences containing an attribute amongst those containing a particular word. i.e. it is an estimate of the co-occurrence of attribute with the word.
Based on attribute model $\hata(\cdot)$ we can define an attribute model $\hatA(\cdot)$ for any sentence $s = \{w_1,\dots,w_k\}$ as follows:
\begin{defn}[Attribute model $\hatA(s)$ for a sentence $s \in W^k$]
\label{defn: attribute model for sentences}
\small
\begin{align}
    \hatA(s = \{w_1,\dots,w_k\}) &:= \max_{w_i\in s} \hata(w_i)\\
    &= \max \{\hata(w_1),\dots,\hata(w_k)\}\label{eqn:attribute model for word}
\end{align}
\normalsize
Note that such a model is conservative and labels a sentence as having an attribute when any word in the sentence has the attribute. For the purpose of attributes such as toxicity, such an attribute model is quite suitable.
\end{defn}


\subsection{Computation of ATE Score of a word with respect to an attribute}

Given a model representing the estimate of the attribute $A$ in a  sentence $s$, denoted as $\widehat\Prob\{A(s) = 1\}$, we can now define the ATE score. 
Note that the Causal ATE score does not depend on the particular model for the estimate $\widehat\Prob\{A(s) = 1\}$ -- i.e. we can use any estimator model.

If we denote $f_A(s)$ as the estimate of $\Prob\{A(s) = 1\}$ obtained from \textit{some} model. We can then define Causal ATE with respect to this estimate. 
If a sentence $s$ is made up of words $\{w_1,\dots,w_i,\dots,w_k\}$. For brevity, given a word $w_i$, from a sentence $s$, we may refer to the rest of the words in the sentence as context $c_i$. Consider a \textit{counter-factual} sentence $s'$ where (only) the $i$th word is changed: $\{w_1,\dots, ,w'_i,\dots,w_k\}$. Such a word $w_i'$ may be the most probable token to replace $w_i$, given the rest of the sentence. 

We now define a certain value that may be called the Treatment Effect ($\TE$), which computes the effect of replacement of $w_i$ with $w'_i$ in sentence $s$, on the attribute probability.

\begin{defn}[Treatment Effect (TE) of a word in a sentence given replacement word]

Let word $w_i$ be replaced by word $w'_i$ in a sentence $s$. Then:
\vspace{-0.12in}
\begin{align}
    \TE(s,w_i,w'_i) 
    &= f_A(s)-f_A(s')\nonumber\\
    &= f_A(\{w_1,\dots,w_i,\dots,w_k\}) \nonumber\\
    &\enspace - f_A(\{w_1,\dots, w'_i,\dots,w_k\})
\end{align}
\vspace{-0.1in}
\end{defn}
The expectation now can be taken over the replacement words, given the context, and over all contexts where the words appear.





\begin{defn}[$\ATE$ of word $w_i$ given dataset $\D$ and an attribute classifier $f(\cdot)$]
\begin{align}
    \ATE(w_i) 
    &= \E_{s\in\D\mid w_i\in s}\left[f(s)-\E_{w_i'\in W} [f(s')]\right]\label{eqn: ATE Score}
\end{align}
where $s'$ is the sentence $s$ where word $w_i$ is replaced by $w_i'$
\end{defn}

This $\ATE$ score precisely indicates the intervention effect of $w_i$ on the attribute probability of a sentence. Notice that this score roughly corresponds to the \textit{expected difference in attribute on replacement} of word.

Now say we compute the $\ATE$ scores for every token $w$ in our universe $W$ in the manner given by \hyperref[eqn: ATE Score]{Equation \ref{eqn: ATE Score}}. We can store all these scores in a large lookup-table. Now, we are in a position to compute an attribute score given a sentence.

\subsection{Computation of Attribute Score for a sentence}

The causal ATE approach suggests that we can build towards the ATE of a sentence given the ATE scores of each of the words in the sentence recursively. 
We illustrate this approach in \hyperref[fig:Causal Graph Illustration for Attributes]{Figure \ref{fig:Causal Graph Illustration for Attributes}}. First, note that each word $w_t$ is stochastically generated based on words $w_1,\dots,w_{t-1}$ in an auto-regressive manner.
If we denote $\{w_1,\dots,w_{t-1}\}$ as $s_{t-1}$, then we can say the  distribution for $w_t$, is generated from $s_{t-1}$ and the structure of the language. To sample from the probabilistic distribution, we may use an exogenous variable such as $U_t$. 

The attribute $A(s_{t-1})$ of a sentence up to $t-1$ tokens, depends only on  $\{w_1,\dots,w_{t-1}\}\equiv s_{t-1}$. 
We now describe a model for computing attribute $A(s_t)$ from $A(s_{t-1})$ and $\ATE(w_t)$. The larger English causal graph moderates influence of $w_t$ on $A(s_t)$ through the $\ATE$ score of the words. We consider $A(s_t) = \max(A(s_{t-1}),\ATE(w_t))$. This is  equivalent to

\vspace{-0.15in}
\begin{align}\label{eqn:ATE of sentence}
    A_{\infty}(s = \{w_1,\dots,w_n\}) = \max_{i\in[n]} \ATE(w_i)
\end{align}
\vspace{-0.05in}

More generally, we propose an attribute score $A(s)$ for this sentence given by $A(s) = \|\{\ATE(w_1),\dots,\ATE(w_n)\}\|_p$ where $\|\cdot\|_p$ indicates the $L_p$-norm of a vector. We can call these attribute scores $A(s)$ as the $\ATE$ scores of a sentence. 

\begin{figure}[!thb]
    \centering
    \includegraphics[width=0.44\textwidth]{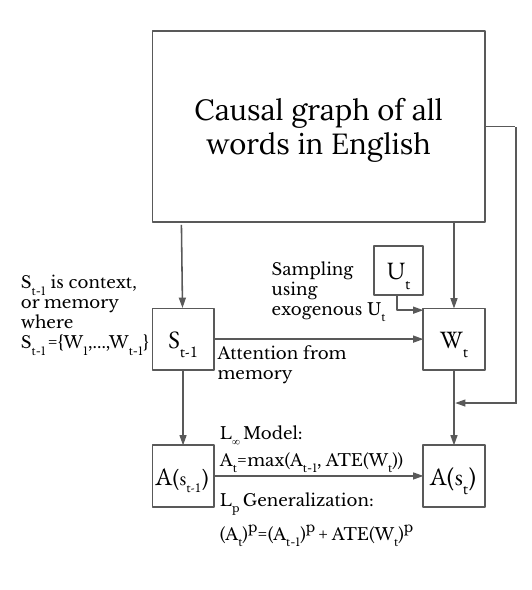}
    \caption{An Illustration of the Causal Graph used to compute the attribute score of a sentence recursively.}
    \label{fig:Causal Graph Illustration for Attributes}
\end{figure}

\section{Theory and Background}
\label{sec:theory}

\begin{table*}[ht]
\small
\centering
\caption{ATE Scores vs Classifier Predictions for different models by Protected Category}
\label{table:metrics for ATE vs Classifier Predictions}
\begin{tabular}{l|ccc|ccc|ccc|ccc}
\toprule
Group $\rightarrow$ & \multicolumn{3}{c|}{\textbf{African}} & \multicolumn{3}{c|}{\textbf{Black}} & \multicolumn{3}{c|}{\textbf{Female}} & \multicolumn{3}{c}{\textbf{Gay}} \\
\midrule
Model $\downarrow$ & Pred & ATE & Diff & Pred & ATE & Diff & Pred & ATE & Diff & Pred & ATE & Diff \\
\midrule
LR & 0.201 & 0.099 & \textcolor{darkcyan}{0.102} & 0.300 & 0.108 & \textcolor{darkcyan}{0.192} & 0.270 & 0.167 & \textcolor{darkcyan}{0.103} & 0.470 & 0.167 & \textcolor{darkcyan}{0.303} \\
SVM & 0.282 & 0.062 & \textcolor{darkcyan}{0.220} & 0.282 & 0.052 & \textcolor{darkcyan}{0.230} & 0.301 & 0.082 & \textcolor{darkcyan}{0.219} & 0.371 & 0.154 & \textcolor{darkcyan}{0.217} \\
GB & 0.225 & 0.052 & \textcolor{darkcyan}{0.173} & 0.335 & 0.071 & \textcolor{darkcyan}{0.264} & 0.225 & 0.000 & \textcolor{darkcyan}{0.225} & 0.653 & 0.204 & \textcolor{darkcyan}{0.449} \\
NB & 0.460 & 0.002 & \textcolor{darkcyan}{0.458} & 0.510 & 0.047 & \textcolor{darkcyan}{0.463} & 0.444 & 0.004 & \textcolor{darkcyan}{0.440} & 0.657 & 0.107 & \textcolor{darkcyan}{0.550} \\
NN1Layer & 0.000 & 0.003 & \textcolor{red}{-0.003} & 0.000 & 0.059 & \textcolor{red}{-0.059} & 0.000 & 0.024 & \textcolor{red}{-0.024} & 1.000 & 0.197 & \textcolor{darkcyan}{0.803} \\
NN2Layer & 0.000 & 0.000 & \textcolor{darkcyan}{0.000} & 0.000 & 0.096 & \textcolor{red}{-0.096} & 0.002 & 0.000 & \textcolor{darkcyan}{0.002} & 1.000 & 0.217 & \textcolor{darkcyan}{0.783} \\
NN3Layer & 0.000 & 0.160 & \textcolor{red}{-0.160} & 0.000 & 0.097 & \textcolor{red}{-0.097} & 0.000 & 0.000 & \textcolor{darkcyan}{0.000} & 0.993 & 0.165 & \textcolor{darkcyan}{0.828} \\
\bottomrule
\end{tabular}
\end{table*}

\vspace{-0.12in}
Now that we have laid the groundwork, we can make proceed to make the central claims of this work.
\begin{lem}
\label{lem:lemma relating sentence to word ATE}
    Consider sentence $s =\{w_1,\dots,w_k\}$. We will make two simple claims:
    \begin{enumerate}
        \item If $\nexists w_i \in s$ such that $\ATE(w_i) \geq c$, then, $A(s) < c$.
        \item If $\exists w_i \in s$ such that $\ATE(w_i) \geq c$, then, $A(s) \geq c$.
    \end{enumerate}
This lemma is straightforward to prove from Definition \ref{eqn:ATE of sentence}.
\end{lem}

\vspace{-0.05in}
We will now make a claim regarding the $\ATE$ score of the given words themselves. Recall that $c_i$ is the context for the word $w_i$ from a sentence $s$. Given $c_i$, $w_i$ is replaced by $w'_i$ by a perturbation model (through Masked Language Modelling).

\noindent Towards our proof, we will make two assumptions:
\begin{assumption}\label{a1:replacement}
    We make a mild assumption on this replacement process: $\hata(w'_i) < \hatA(c_i)$. Grounding this in the attribute of toxicity, we can say that the replacement word is less toxic than the context. This is probable if the replacement model has been trained on a large enough corpus. See \cite{madhavan-etal-2023-cfl} for empirical results showing this claim to be true in practice.
\end{assumption}
\begin{assumption}\label{a2:dataset}
    We make an assumption on the dataset. 
A \textit{spurious correlate} has a word with a higher attribute score in the rest of the sentence for sentences labelled as having the attribute. \noindent For example, in the case of toxicity, a spurious correlate like Muslim, has a more toxic word in the rest of the sentence, when the sentence is labelled as toxic.
\end{assumption}

\vspace{-0.05in}
\noindent Given these assumptions, we have the following theorem:

\begin{mdframed}[backgroundcolor=azureblue]
\begin{thm}
\label{thm: ATE less than 0.25}
Given Assumptions \ref{a1:replacement} and \ref{a2:dataset} for  a spurious correlate $w_i$, \enspace $\ATE(w_i) \leq 0.25$.
\end{thm}
\end{mdframed}
\begin{proof}
    If we consider three variables $\{\hatA(c_i),\hata(w_i),\hata(w'_i)\}$, there are six possible orderings of this set. We can subsume these orderings into two cases: \\(1) $\hatA(c_i) < \hata(w'_i)$ and (2) $\hatA(c_i) \geq \hata(w'_i)$. Within these cases, we study the variation of $ATE(w_i)$ with $\hata(w_i)$. We plot these in the Figure \ref{fig:Graphs for ATE Score}. Using a case-by-case analysis over these possibilities, we prove the statement. The full proof of the Theorem is provided in Appendix \ref{sec:proof of theorem 1}.
\end{proof}    

\begin{figure}[!bht]
\begin{center}
\includegraphics[width=0.48\textwidth]{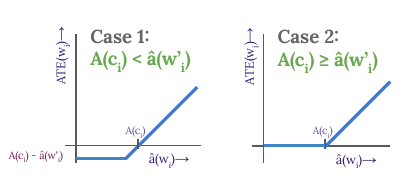}
\caption{Graph of ATE score of a given word $w_i$ with $\hata(w_i)$ given two cases }
\label{fig:Graphs for ATE Score}
    
\end{center}
\end{figure}

\noindent In the following section we provide experimental justification for our work through experimental results.

\vspace{-0.12in}
\section{Experimental Work}
\label{section: experiments}

\vspace{-0.05in}
The primary focus of our experimental assessments is to compare classifier predictions and the Average Treatment Effect (ATE) scores computed from these classifiers using the dataset provided by \citep{gao2017detecting, zampieri-etal-2019-semeval} for bias inducing words that may include protected groups.

We provide justification of our central claim that \textit{causal ATE mitigates bias} in this section through two experiments shown in Fig. \ref{plot:Multiple Datasets} and  Table \ref{table:metrics for ATE vs Classifier Predictions}. In Figure \ref{plot:Multiple Datasets} we compare the bias-mitigation performance of the ATE score compared across two datasets, for protected groups. In the second experiment, Table \ref{table:metrics for ATE vs Classifier Predictions}, over a single dataset \citep{gao2017detecting}, we compare the ATE scores computed from various classifiers with the classifier predictions. 

Together, these experiments show that across datasets and models, our ATE based classification provides lower than 0.25 toxicity score for protected groups. Moreover, it reduces the toxicity in the classifier significantly as noticeable in Figure \ref{fig:Multiple Models plot}. We provide the full code in \href{https://github.com/causalate-mitigates-bias/causal-ate-mitigates-bias}{our anonymous GitHub repository}.


\vspace{-0.12in}
\section{Discussion on Generalizability}
\label{sec:conclusions}
\vspace{-0.1in}

While we our experimental results have pertained to the use of Causal ATE as a metric for mitigating bias in toxicity classification, our theoretical results extend to any language attributes. In fact, in Appendix Section \ref{section: Causal ATE is Generalizable}, we showcase different style attributes to which such an analysis can be applied. We hope that such causal approaches can be utilized for general use cases such as style control using LLMs.

In conclusion, our work provides a theoretical justification for using the 
causality-based concepts of counterfactuals, and ATE scores for controlled text generation. We provide experimental results that validate these claims. 
We show that the simple perturbation-based method of Causal ATE removes the unintended bias effect through reduction of false positives, additionally making systems more robust to biased data. 


\section{Limitations}
\label{sec: Limitations}

The limitations of our proposed framework are described in detail in this section. 

\noindent\textbf{1. Owing to Pre-trained models:}
Third-party hatespeech detectors such as \HateBERT tend to overestimate the prevalence of toxicity in texts having mentions of minority or protected groups due to sampling bias, or just spurious correlations \citep{paz2020hate, waseem2016you, dhamala2021bold}. ATE computation though following causal mechanisms rely on these detectors for initial  attribute probability scores. Additionally, these models suffer from low annotator agreement during dataset annotation because of absence of concrete defining hatespeech taxonomy \citep{sap2019risk}. Causal nature of our approach tends to mitigates bias but not completely eliminated the problem. 

\noindent\textbf{2. Owing to language and training corpus:}
We showcase empirically the utility of our theoretical claims in this study and conducted monolingual experiments on English language which could be further extended to other languages. Additionally, training corpora used for training \HateBERT and MLM model are known to contain curated data from internet, where reliability and factual accuracy is a known issue \citep{gehman2020realtoxicityprompts}. Hence, we are limited by the distributions of our training corpora in terms of what the model can learn and infer. 

\noindent\textbf{3. Owing to distribution shift between datasets:}
There are limitations that get introduced due to change in vocabulary from training to test sets. Sometimes, words which occur in test set are not in ATE training set, we ignore such words but could impact downstream perfomance of LLM if word was important. In case of such a distribution shift between the datasets, our model may not work as expected.

\section{Ethics Statement}
\label{sec: Ethics Statement}
Our paper addresses the crucial issue of bias and toxicity in language models by using causal methods that involve several ethical concerns, that we address herein:

\noindent \textbf{1. Monolingual limitation :} This work addresses the problem of mitigation of toxicity in Language models (LMs) for English language, even though there more than 7000 languages globally \citep{joshi2020state} and future works should address more generalizable and multilingual solutions so that safety is promised for diverse set of speakers and not limited to English speakers \citep{weidinger2022taxonomy}

\noindent \textbf{2. No one fixed toxicity taxonomy:} Literature survey  highlights the fact that toxicity, hate and abuse and other related concepts are loosely defined and vary based on demographics and different social groups \citep{paz2020hate, yin2021towards}. Henceforth, affecting the quality of hatespeech detection systems (\HateBERT) used in this work. These variations differences between cultural definitions of toxicity poses an ethical challenge \citep{jacobs2021measurement, welbl2021challenges}. 

\noindent \textbf{3. Third party classifiers for toxicity detection:} Reliance on the third party classifiers for toxicity detection can itself beat the purpose of fairness as these systems are reported to be biased towards certain protected groups and overestimate the prevelence of toxicity associated with them in the texts \citep{davidson2019racial, abid2021large, hutchinson2020social, dixon2018measuring, sap2019risk}. For most part, we take care of these by using causal mechanisms but the ATE computation still involves using a toxicity classifier (\HateBERT) model.

\section{Potential Risks}
Any controlled generation method runs the runs the risk of being reverse-engineered, and this becomes even more crucial for detoxification techniques. In order to amplify their ideologies, extremists or terrorist groups could potentially subvert these models by prompting them to generate extremist, offensive and hateful content \citep{mcguffie2020radicalization}.



\section{References}

\bibliography{References}

\clearpage


\appendix
\onecolumn

\section{Related Works}
\label{sec:related-work}

\textbf{Controlled Generation} can be broadly categorized into fine-tuning methods \cite{krause2020gedi}, data-based \cite{keskar2019ctrl, gururangan2020don}, decoding-time approaches using attribute classifiers \cite{dathathri2019plug, krause2020gedi} and causality based approaches \cite{madhavan-etal-2023-cfl}. Majority of these techniques were tested on toxicity mitigation and sentiment control. 
The dependence of attribute regularizers on probabilistic classifiers make them prone to  such spurious correlations \cite{kaddour2022causal, feder2022causal}.

\noindent In the \textbf{Unintended Bias problem} LMs which are detoxified inherit a tendency to be biased against protected groups. LM quality is compromised due to a detoxification side-effect \cite{welbl2021challenges, xu2021detoxifying}.  
Some works address LM control through improving datasets \cite{sap2019social}. 
Unfortunately, this makes annotation and data curation more expensive. As an alternative, there is growing interest in training accurate models in presence of biased data \cite{oren2019distributionally}. Our work fits into this framework. 

\noindent In the context of \textbf{Toxicity Mitigation}, \cite{welbl2021challenges} highlight that detoxification methods have unintended effects on marginalized groups. 
They showcased that detoxification makes LMs more brittle to distribution shift, affecting its robustness in certain parts of language that contain mentions of minority groups. 
Concretely, words such as “female” are identified as being toxic, as they co-occur with toxic text, and hence the LM stops speaking about them \cite{xu2021detoxifying}. This is called the unintended bias problem. This unintended bias problem can manifest as systematic differences in performance of the LM for different demographic groups.

\noindent  \textbf{Toxicity Detection} Toxicity is a well studied problem in context of responsible and safe AI effort. Hence, we foucs our experiments on toxicty mitigation in this study. Several works have also studied the angle from toxic text detection. Numerous studies have explored toxic text detection, including \HateBERT \citep{caselli2020hatebert}, \HateCheck \citep{rottger2020hatecheck}, and \perspective \cite{lees2022new}. We employ the \HateBERT model for assessing local hatefulness and utilize \perspective for third-party evaluation, where we report the corresponding metrics.

\vspace{0.1in}
\noindent \textbf{Causal Methods for text:} Spurious correlations between protected groups and toxic text can be identified is by understanding the causal structure.  \cite{feder2022causal} emphasizes on the connect between causality and NLP. 
Towards mitigation of the bias problem \cite{madhavan-etal-2023-cfl} proposed the use of Causal ATE as a regularization technique and showed experimentally that it does indeed perform as intended. 
In this paper, we probe the Causal ATE metric theoretically, and prove that the Causal ATE metric is less susceptible to false positives. An  attribute control method based on this metric would mitigate unintended bias. We provide a theoretical basis from which to understand the Causal ATE metric and showcase that this causal technique provides robustness across contexts for attribute control in language models.

\vspace{-0.1in}
\section{Importance of using a Causal Graph}

\vspace{-0.1in}
Given estimates of the probability $\Prob\{a_i\mid s\}$ for attributes in text generated by a Language Model (LM), the potential for fine-tuning the LM towards specific attributes becomes apparent. However, numerous challenges persist.

\vspace{0.05in}
\noindent Firstly, attribute classifiers are prone to spurious correlations. For instance, if a protected token like `Muslim' frequently appears in toxic sentences, the attribute classifier detecting toxicity might penalize the generation of the word `Muslim'. This brings out in light that there is a trade-off between detoxification of LM and LM quality for text generation clearly detailed out in \citep{welbl2021challenges}. LM avoids to generate sentences containing protected tokens leading to higher perplexity for texts with these protected attrbiutes. 
Additionally, these classifier models providing $\Prob{a_i\mid s}$ estimates themselves may be LMs, resulting in slow training and requiring substantial computational resources.

\vspace{0.05in}
\noindent Utilizing a causal graph directly addresses these challenges. It offer computational efficiency during training and are immune to spurious correlations, detecting interventional attribute distributions rather than conditional distributions through counterfactual interventions. Moreover, we get both flexibility and transparency regarding their exact form, features unavailable with LM classifiers.

\section{Causal ATE is Generalizable}
\label{section: Causal ATE is Generalizable}
\begin{figure}[!htb]
\begin{center}
    \includegraphics[width=0.72\textwidth]{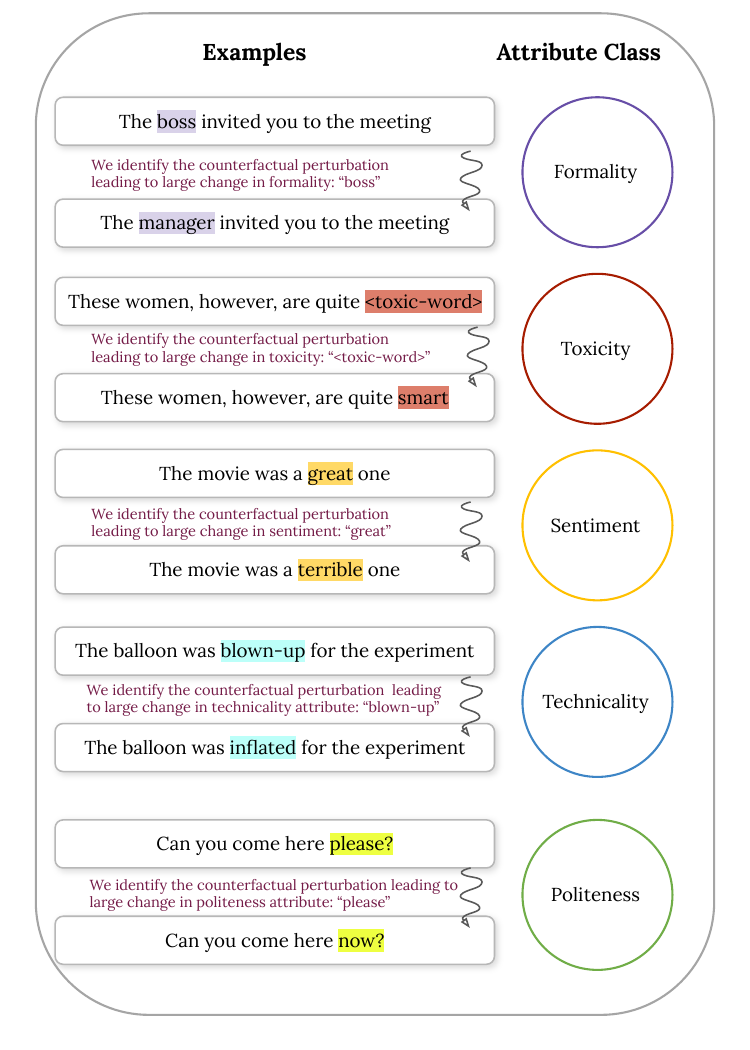}
    \caption{Illustration of how perturbation of the words in a sentence may be used to identify the most important words with respect to an attribute.}
    \label{fig:Page1 Illustration}
\end{center}
\end{figure}

\begin{mdframed}[backgroundcolor=azureblue] 
While the main sections in the paper consider the attribute class of toxicity, we illustrate here that this method can equally be used for various attribute classes thereby easily scalable and generalizable. For instance, in the case of a style like formality, changing `boss' to `manager' has changes the sentence attribute to being more formal. Similarly, a change from the word `terrific' or `great' to `terrible' in the context of a movie review, changes the entire meaning of a sentence, and effectively conveys a more negative sentiment.

\vspace{0.15in}
\noindent Similarly, simple word changes can lead to the language being more technical or polite. Figure \ref{fig:Page1 Illustration} illustrates that causal ATE can be used across various attributes for bias mitigation. The underlying idea is that we can perturb particular words in their context to check the change that they cause on the desired attribute.
\end{mdframed}

\clearpage
\section{Proof of Theorem 1}
\label{sec:proof of theorem 1}

\begin{thm*}
\label{Appendixthm: ATE less than 0.25}
Given Assumptions \ref{a1:replacement} and \ref{a2:dataset}, for $w_i$ which is a spurious correlate, \enspace $\ATE(w_i) \leq 0.25$.
\end{thm*}
\begin{proof}
    If we consider three numbers $\{\hatA(c_i),\hata(w_i),\hata(w'_i)\}$, there are six possible orderings of this set. We can subsume these orderings into two cases: 
    \vspace{0.1in}
    \begin{enumerate}
        \item $\hatA(c_i) < \hata(w'_i)$.
        \item $\hatA(c_i) \geq \hata(w'_i)$.
    \end{enumerate}
    
    \noindent Within these cases, we study the variation of $ATE(w_i)$ with $\hata(w_i)$. We plot these results in the Figure \ref{Appendixfig:Graphs for ATE Score}.
    
    \begin{figure}[!thb]
    \begin{center}
    \includegraphics[width=0.72\textwidth]{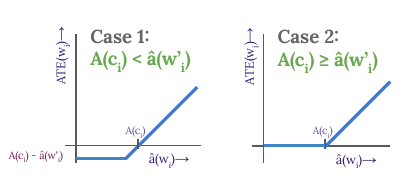}
    \caption{Graph of ATE score of a given word $w_i$ with $\hata(w_i)$ given two cases }
    \label{Appendixfig:Graphs for ATE Score}
        
    \end{center}
    \end{figure}
    \noindent Note that by Assumption \ref{a1:replacement}, we have $\hata(w'_i) \leq \hatA(c_i)$. Therefore, Case (2) in Figure \ref{Appendixfig:Graphs for ATE Score} is sufficient for proof.
    We have:
    \begin{align}
        \ATE(w_i)& = \E_{s\in\D} \E_{w'_i\in s'} \left[\hatA(s) - \hatA(s')\right]\\
        &= \frac{n(A(s)=1 \mid w_i \in s)}{n(s \mid w_i \in s)} \E_{w'_i\in s'} \left[\hatA(s) - \hatA(s')\right] \nonumber \\
        &+ \frac{n(A(s)=0 \mid w_i \in s)}{n(s \mid w_i \in s)} \E_{w'_i\in s'} [\hatA(s) - \hatA(s')]
     \end{align}
     But by Assumption \ref{a2:dataset}, in toxic sentences, $\hatA(s) = \hatA(c_i) \geq \hata(w'_i)$. Therefore $\E_{w'_i\in s'} \{\hatA(s) - \hatA(s')\} = 0$. Then:
     \begin{align}
         \ATE(w_i) &= \frac{n(A(s) = 0  \mid w_i \in s)}{n(s \mid w_i \in s)} \E_{w'_i\in s'} [\hatA(s) - \hatA(s')]
    \end{align}

\noindent But $\hatA(s) - \hatA(s')$ is at most $\hata(w_i)$ as:\\ (1) if $\hata(w_i) \leq \hatA(c_i)$, then $\hatA(s) - \hatA(s') = 0$ \\(2) otherwise $\hatA(s) - \hatA(s') = \hata(w_i) - \hatA(s') \leq \hata(w_i)$. Then:
    \begin{align}        
        \ATE(w_i) &\leq \frac{n(A(s)=0 \mid w_i \in s)}{n(s \mid w_i \in s)} \hata(w_i) \\
        &= \frac{n(A(s)=0 \mid w_i \in s)}{n(s \mid w_i \in s)} \frac{n(A(s)=1 \mid w_i \in s)}{n(s \mid w_i \in s)}\nonumber\\
        &= p\cdot(1-p)
    \end{align}
     for some $p\in[0,1]$. But $p\cdot(1-p)\leq 0.25\enspace \forall p\in [0,1]$.
\end{proof}

\vspace{-0.05in}
\noindent Based on Theorem \ref{Appendixthm: ATE less than 0.25} and Lemma \ref{lem:lemma relating sentence to word ATE}, $A(s) \leq 0.25$ 
if each $w_i\in s$\hspace{1.5pt}is a spurious correlate,\hspace{1.5pt}i.e.\hspace{2pt}non-causal,\hspace{2pt}for attribute $A.$

\section{Experimental Results in Detail for Zampieri et al. and Gao et al. Datasets}

In this section we provide the full set of results on our runs across models for the two datasets \citet{gao2017detecting} and \citet{zampieri-etal-2019-semeval}. The plot in \ref{fig:Multiple Models plot} illustrates the reduction in toxicity classification by using ATE score on the \citet{zampieri-etal-2019-semeval} dataset for three types of classifiers. 

\noindent We provide the full tabular results in Tables \ref{table: fullresults gao} and \ref{table: fullresults zampieri}.

\begin{figure}[!thb]
\begin{center}
\includegraphics[width=\textwidth]{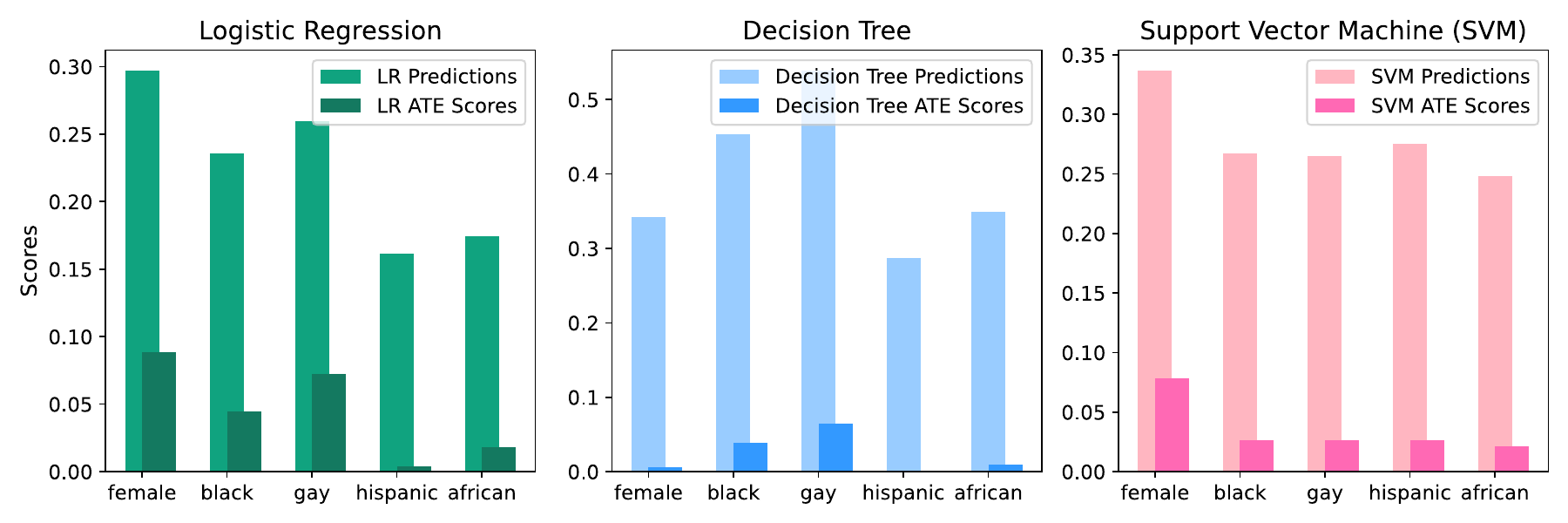}
\caption{For the \citet{zampieri-etal-2019-semeval} dataset, we compute the mitigation of toxicity score using three different classifiers, and the ATE scores computed using the respective classifiers. These show a reduction on toxicity for protected groups across different models.}
\label{fig:Multiple Models plot}   
\end{center}
\end{figure}


\vspace{0.5in}
\begin{table*}[!hb]
\tiny 
\centering
\caption{Classifier Metrics by Protected Category for the Gao et al. Dataset}
\label{table: fullresults gao}
\begin{tabular}{l|ccc|ccc|ccc|ccc|ccc}
\toprule
 & \multicolumn{3}{c|}{African} & \multicolumn{3}{c|}{Black} & \multicolumn{3}{c|}{Female} & \multicolumn{3}{c|}{Gay} & \multicolumn{3}{c}{Hispanic} \\
\midrule
 & Pred & ATE & Diff & Pred & ATE & Diff & Pred & ATE & Diff & Pred & ATE & Diff & Pred & ATE & Diff \\
\midrule
LR & 0.201 & 0.099 & \textcolor{darkcyan}{0.102} & 0.300 & 0.108 & \textcolor{darkcyan}{0.192} & 0.270 & 0.167 & \textcolor{darkcyan}{0.103} & 0.470 & 0.167 & \textcolor{darkcyan}{0.303} & 0.166 & 0.011 & \textcolor{darkcyan}{0.155} \\
SVM & 0.282 & 0.062 & \textcolor{darkcyan}{0.220} & 0.282 & 0.052 & \textcolor{darkcyan}{0.230} & 0.301 & 0.082 & \textcolor{darkcyan}{0.219} & 0.371 & 0.154 & \textcolor{darkcyan}{0.217} & 0.246 & 0.057 & \textcolor{darkcyan}{0.189} \\
GB & 0.225 & 0.052 & \textcolor{darkcyan}{0.173} & 0.335 & 0.071 & \textcolor{darkcyan}{0.264} & 0.225 & 0.000 & \textcolor{darkcyan}{0.225} & 0.653 & 0.204 & \textcolor{darkcyan}{0.449} & 0.225 & 0.020 & \textcolor{darkcyan}{0.205} \\
NB & 0.460 & 0.002 & \textcolor{darkcyan}{0.458} & 0.510 & 0.047 & \textcolor{darkcyan}{0.463} & 0.444 & 0.004 & \textcolor{darkcyan}{0.440} & 0.657 & 0.107 & \textcolor{darkcyan}{0.550} & 0.615 & 0.000 & \textcolor{darkcyan}{0.615} \\
NN1Layer & 0.000 & 0.003 & \textcolor{red}{-0.003} & 0.000 & 0.059 & \textcolor{red}{-0.059} & 0.000 & 0.024 & \textcolor{red}{-0.024} & 1.000 & 0.197 & \textcolor{darkcyan}{0.803} & 0.000 & 0.000 & \textcolor{darkcyan}{0.000} \\
NN2Layer & 0.000 & 0.000 & \textcolor{darkcyan}{0.000} & 0.000 & 0.096 & \textcolor{red}{-0.096} & 0.002 & 0.000 & \textcolor{darkcyan}{0.002} & 1.000 & 0.217 & \textcolor{darkcyan}{0.783} & 0.000 & 0.000 & \textcolor{darkcyan}{0.000} \\
NN3Layer & 0.000 & 0.160 & \textcolor{red}{-0.160} & 0.000 & 0.097 & \textcolor{red}{-0.097} & 0.000 & 0.000 & \textcolor{darkcyan}{0.000} & 0.993 & 0.165 & \textcolor{darkcyan}{0.828} & 0.000 & 0.000 & \textcolor{darkcyan}{0.000} \\
\bottomrule
\end{tabular}
\end{table*}

\vspace{0.5in}
\begin{table*}[!hb]
\tiny 
\centering
\caption{Classifier Metrics by Protected Category for the Zampieri et al. Dataset}
\label{table: fullresults zampieri}
\begin{tabular}{l|ccc|ccc|ccc|ccc|ccc}
\toprule
 & \multicolumn{3}{c|}{African} & \multicolumn{3}{c|}{Black} & \multicolumn{3}{c|}{Female} & \multicolumn{3}{c|}{Gay} & \multicolumn{3}{c}{Hispanic} \\
\midrule
 & Pred & ATE & Diff & Pred & ATE & Diff & Pred & ATE & Diff & Pred & ATE & Diff & Pred & ATE & Diff \\
\midrule
LR & 0.174 & 0.020 & \textcolor{darkcyan}{0.154} & 0.236 & 0.049 & \textcolor{darkcyan}{0.187} & 0.297 & 0.075 & \textcolor{darkcyan}{0.223} & 0.260 & 0.098 & \textcolor{darkcyan}{0.162} & 0.161 & 0.143 & \textcolor{darkcyan}{0.018} \\
SVM & 0.248 & 0.030 & \textcolor{darkcyan}{0.218} & 0.267 & 0.036 & \textcolor{darkcyan}{0.232} & 0.337 & 0.068 & \textcolor{darkcyan}{0.269} & 0.265 & 0.033 & \textcolor{darkcyan}{0.232} & 0.275 & 0.119 & \textcolor{darkcyan}{0.156} \\
GB & 0.269 & 0.020 & \textcolor{darkcyan}{0.249} & 0.269 & 0.013 & \textcolor{darkcyan}{0.256} & 0.269 & 0.008 & \textcolor{darkcyan}{0.261} & 0.269 & 0.003 & \textcolor{darkcyan}{0.266} & 0.269 & 0.033 & \textcolor{darkcyan}{0.236} \\
NB & 0.349 & 0.009 & \textcolor{darkcyan}{0.341} & 0.453 & 0.055 & \textcolor{darkcyan}{0.398} & 0.343 & 0.183 & \textcolor{darkcyan}{0.160} & 0.539 & 0.070 & \textcolor{darkcyan}{0.469} & 0.287 & 0.000 & \textcolor{darkcyan}{0.287} \\
NN1Layer5 & 0.000 & 0.000 & \textcolor{red}{-0.000} & 0.000 & 0.052 & \textcolor{red}{-0.052} & 0.000 & 0.000 & \textcolor{red}{-0.000} & 0.000 & 0.114 & \textcolor{red}{-0.114} & 0.000 & 0.000 & \textcolor{darkcyan}{0.000} \\
NN2Layer105 & 0.000 & 0.000 & \textcolor{darkcyan}{0.000} & 0.000 & 0.090 & \textcolor{red}{-0.090} & 0.000 & 0.170 & \textcolor{red}{-0.170} & 0.000 & 0.104 & \textcolor{red}{-0.104} & 0.000 & 0.000 & \textcolor{darkcyan}{0.000} \\
NN3Layer20105 & 0.000 & 0.200 & \textcolor{red}{-0.200} & 0.000 & 0.126 & \textcolor{red}{-0.126} & 0.000 & 0.075 & \textcolor{red}{-0.075} & 0.000 & 0.046 & \textcolor{red}{-0.046} & 0.000 & 0.000 & \textcolor{darkcyan}{0.000} \\
\bottomrule
\end{tabular}
\end{table*}

\noindent \textbf{Note:} We note that the neural classifiers may have overfit on the \citet{zampieri-etal-2019-semeval} dataset due to which the numbers are either close to 0 or 1.

\clearpage
\section{Experimental Setup}

\subsection{Dataset Details}

We conducted experiments on the publically available Zampieri \citep{zampieri2019semeval} and Gao \citep{gao2017detecting} datasets.

\subsection{Hyper-parameters}
Details in our GitHub repository: \href{https://github.com/causalate-mitigates-bias/causal-ate-mitigates-bias}{github.com/causalate-mitigates-bias/causal-ate-mitigates-bias}

\subsection{Result Statistics}
Our run details are provided on the README.md file of our GitHub repository: \href{https://github.com/causalate-mitigates-bias/causal-ate-mitigates-bias/blob/main/README.md}{https://github.com/causalate-mitigates-bias/causal-ate-mitigates-bias/blob/main/README.md}

\subsection{Compute Resources}

All our experiments were carried out using NVidia 1080 GPU Machines with Intel Core i7-7700K @ 4.2GHz. Our experiments utilized approximately 100 CPU-hours and 10 GPU-hours.

\subsection{Tools and packages}
 We list the tools used in our requirements.txt file of our GitHub repository: \href{https://github.com/causalate-mitigates-bias/causal-ate-mitigates-bias/blob/main/requirements.txt}{https://github.com/causalate-mitigates-bias/causal-ate-mitigates-bias/blob/main/requirements.txt}

\subsection{Use of AI Assistants}

We have used AI Assistants (GPT-4) to help format our charts as well as help create latex tables.

\end{document}